\documentclass[letterpaper]{article} 
\usepackage{aaai2026}  
\usepackage{times}  
\usepackage{helvet}  
\usepackage{courier}  
\usepackage[hyphens]{url}  
\usepackage{graphicx} 
\usepackage{natbib}  
\usepackage{caption} 
\usepackage{nicefrac}  
\usepackage{dsfont}    
\usepackage{amsfonts}  
\usepackage{amsmath}   
\usepackage{amsthm}
\usepackage{appendix}
\usepackage{booktabs}
\usepackage{thmtools}
\usepackage[T1]{fontenc}         
\usepackage{microtype}           
\usepackage{cleveref}            

\newcommand{\R}{\mathds{R}}
\newcommand{\E}{\mathds{E}}
\DeclareMathOperator*{\argmax}{arg\,max}
\DeclareMathOperator*{\argmin}{arg\,min}
\DeclareMathOperator*{\rowmax}{rowmax}
\DeclareMathOperator*{\colmax}{colmax}
\DeclareMathOperator*{\relu}{relu}
\DeclareMathOperator*{\softmax}{softmax}

\usepackage{thm-restate}
\declaretheoremstyle[
    style=definition, 
    bodyfont=\normalfont\itshape,
    numberwithin=section
]{mystyle}
\declaretheoremstyle[
    headfont={\it}, 
    bodyfont=\normalfont,
]{myproofstyle}

\declaretheorem[style=mystyle]{theorem}
\declaretheorem[style=mystyle, sibling=theorem]{definition, example, lemma, proposition, remark, corollary, conjecture}

\title{ElementaryNet: A Non-Strategic Neural Network for\\Predicting Human Behavior in Normal-Form Games}
\author{
    Greg d'Eon\textsuperscript{\rm 1},
    Hala Murad\textsuperscript{\rm 1},
    Kevin Leyton-Brown\textsuperscript{\rm 1},
    James R. Wright\textsuperscript{\rm 2},
}
\affiliations{
    \textsuperscript{\rm 1}University of British Columbia
    \textsuperscript{\rm 2}University of Alberta\\
    gregdeon@cs.ubc.ca, hmurad01@student.ubc.ca, james.wright@ualberta.ca, kevinlb@cs.ubc.ca
}

\begin{document}

\maketitle

\begin{abstract}
Behavioral game theory models serve two purposes: yielding insights into how human decision-making works, and predicting how people would behave in novel strategic settings.
A system called GameNet represents the state of the art for predicting human behavior in the setting of unrepeated simultaneous-move games, combining a simple ``level-$k$'' model of strategic reasoning with a complex neural network model of non-strategic ``level-0'' behavior.
Although this reliance on well-established ideas from cognitive science ought to make GameNet interpretable, the flexibility of its level-0 model raises the possibility that it is able to emulate strategic reasoning.
In this work, we prove that GameNet's level-0 model is indeed too general.
We then introduce ElementaryNet, a novel neural network that is provably incapable of expressing strategic behavior.
We show that these additional restrictions are empirically harmless, with ElementaryNet and GameNet having statistically indistinguishable performance.
We then show how it is possible to derive insights about human behavior by varying ElementaryNet's features and interpreting its parameters, finding evidence of iterative reasoning, learning about the depth of this reasoning process, and showing the value of a rich level-0 specification.
\end{abstract}

\begin{links}
    \link{Code \& data}{https://github.com/gregdeon/elementarynet}
\end{links}

\section{Introduction}
Human behavior in strategic settings often deviates significantly from the predictions of classical game theory: for instance, humans often play dominated actions (as in the unrepeated Prisoner's Dilemma), which a fully rational agent would not.
Behavioral game theory aims to address these shortcomings by developing predictive models of human strategic behavior. 
The purpose of such models is two-fold.
First, such models that have been trained on experimental data can be analyzed to yield insight about human psychology. 
In this sense, classical game theory can be seen as a zero-parameter model that predicts poorly; 
the goal here is to find models that make much more accurate predictions, but are as simple and interpretable as possible.
Second, we might simply build models that are as accurate as possible. Such models are useful for developing agents that will interact with humans or for designing systems of rules (``mechanisms'') that will perform well when faced with human participants. 

Past work has made substantial progress towards the first goal, demonstrating that so-called iterative reasoning models with just two parameters are surprisingly good at predicting behavior in new strategic settings, and hence giving insight into the way human subjects reason. 
Let us describe Quantal Cognitive Hierarchy (QCH), the best performing such model~\citep{Camerer2004}. 
QCH predicts that people perform strategic reasoning at some finite ``level''. 
The model includes a probability distribution describing the proportion of agents at each level; when this distribution is Poisson (having one parameter), we call the model QCHp. 
Level-0 reasoners are nonstrategic: they perform some arbitrary computation that falls short of forming beliefs about their opponents and best responding to these beliefs. 
\citet{Camerer2004} simply assert that level-0 reasoners choose actions uniformly at random, which is clearly nonstrategic (and parameter free); we call this instantiation of the model Uniform + QCHp. 
Level-1 reasoners quantally best respond to level-0 agents, where quantal response is a noisy version of best response that depends on a ``precision'' parameter. 
Level-$k$ reasoners for $k \geq 2$ are more complex: they know the distribution over levels $0, \ldots, k-1$, and quantally best respond to the aggregate distribution of play by all lower-level agents. 
More recent work showed that even better performance can be achieved by hand-crafting a richer level-0 model based on insights from cognitive psychology, at the expense of adding more parameters \citep{JAIR}.

Regarding the second goal of raw predictive performance, it is perhaps unsurprising that the current state-of-the-art model for predicting human behavior in unrepeated normal-form games, dubbed GameNet~\citep{GameNet}, is based on a neural network.
GameNet extends the Uniform + QCHp architecture by replacing the uniform level-0 specification with a learned model and the Poisson distribution by a finite-depth histogram.
The learned model uses a permutation-equivariant architecture to express the inductive bias that an action should be played with the same probability regardless of its index in the game matrix, but is otherwise unrestricted. 
Empirically, GameNet dramatically outperforms Uniform + QCHp. 

Although its neural network component is an uninterpretable black box, GameNet was based on the QCH architecture in part to allow for meaningful interpretation: one can ask about the shape of the level distribution, the precision parameter, and the way model performance varies when levels are added or removed.
In this vein, \citet{GameNet} made a striking observation: GameNet performed best when it was restricted to predict that all agents were level-0. 
What should we make of this result? 
Does it show that iterative strategic reasoning is a dead end, and that human subjects are better described in some different way? 
Or, does it show that a sufficiently general neural network is already able to simulate iterative strategic reasoning, and hence that GameNet's apparent interpretability was a mirage? 
Finally, how would we tell the difference?

In this work, we make three main contributions.
First, in \Cref{sec:gamenet-strategic}, we show that GameNet's purportedly level-0 neural network \emph{is} capable of strategic reasoning, and is hence inappropriate to use for describing level-0 agents. 
Our proof is constructive: we give a specific setting of its parameters that computes quantal best response to maxmax, a strategic model.
This finding helps to explain why adding levels of strategic reasoning did not improve GameNet's performance.

Second, in \Cref{sec:enet}, we introduce \emph{ElementaryNet}, a new restriction of GameNet's architecture that is only capable of non-strategic behavior.
Our architecture is rooted in the concept of elementary models~\cite{WLBnonstrategic}, a set of behavioral models which are provably incapable of representing strategic behavior.
Elementary models compute a real-valued ``potential'' for each outcome of the game that summarizes the set of utilities for all the players into a single number, and then compute a distribution of play using only those potentials.
Intuitively, this restriction to a single summary per outcome prevents the model from forming a belief about the opponent's play (based on the opponent's utilities), and then best responding to that belief (based on the model's own utilities). 
Our architecture is a convex combination of neural networks whose inputs are summarized by a single potential function each; it is thus a convex combination of elementary models, which is provably non-strategic in a precise, formal sense.

Finally, in \Cref{sec:experiments}, we perform extensive experiments on ElementaryNet.
We first show that, despite the additional restrictions on its architecture, ElementaryNet combined with an iterative reasoning model achieves predictive power statistically indistinguishable from that of GameNet.
We then show how the model's clean delineation between strategic and non-strategic components makes it possible to obtain interpretable insights about human behavior by varying different features of the model.
In particular, we show that ElementaryNet performs significantly worse without a strategic model, demonstrating that iterative reasoning is indeed a good model of human behavior.  
We show that restricting the level-0 model---to only consider the player's own payoffs, or to use a fixed set of four basis potentials---degrades performance, demonstrating the value of a rich level-0 specification.
Finally, we successfully interpret the parameters of the iterative reasoning model that are co-learned with ElementaryNet, in exactly the same way that was uninformative with GameNet.

\section{Preliminaries}
We begin by fixing notation.
A \textit{2-player $n \times m$ normal-form game} is a tuple $G = (A, u)$, where $A = A_1 \times A_2$ is the set of \textit{action profiles}; $A_1 = \{1, \dots, n\}$ and $A_2 = \{1, ..., m\}$ are the sets of \textit{actions} available to agents 1 and 2, respectively; and $u = (u_1, u_2)$ is profile of \textit{utility functions} $u_i: A \to \R$, each mapping an action profile to a real-valued utility for agent $i$.
For convenience, we will sometimes refer to the utility matrices $U^1 = [u_1(a_i, a_j)]_{ij}$ and $U^2 = [u_2(a_i, a_j)]_{ij}$.

A \textit{behavior} $s_i$ is an element of the simplex $\Delta(A_i)$, representing a distribution over agent $i$'s actions; we use the non-standard term ``behavior'', rather than the more standard ``strategy'', to avoid other awkward terminology, such as a ``non-strategic strategy''.
A \textit{behavior profile} is a tuple of behaviors $s = (s_1, s_2)$.
Overloading notation, we denote an agent's expected utility as $u_i(s)$ = $\E_{a \sim s} u_i(a)$.
For either agent $i$, we write $s_{-i}$ to represent the behavior of the other agent, and $(s_i, s_{-i})$ to refer to a behavior profile.
A \emph{behavioral model} is a function $f_i$ that maps a game $G = (A, u)$ to a behavior $f_i(G) \in \Delta(A_i)$.

\subsection{Existing Behavioral Models}
We now describe several behavioral models used in prior behavioral game theory work.
A common building block in many of these models is the concept of quantal best response.
\begin{definition} 
    The (logit) \emph{quantal best response} to a strategy $s_{-i}$ is 
    $QBR_i(s_{-i}; \lambda, G)(a_i) \propto \exp[\lambda \cdot u_i(a_i, s_{-i})],$
    where $\lambda$ (the \emph{precision} parameter) controls the agent's sensitivity to differences in utilities.
\end{definition}

The quantal cognitive hierarchy (QCH) model~\cite[e.g.,][]{WLB2017} combines quantal best response with a model of iterative reasoning.
    \begin{definition} 
    Let $G$ be a game, $\lambda \in \R$ be a precision, $s^0$ be a profile of level-0 behaviors in $G$, and $D$ be a probability distribution over levels. 
    Then, the \emph{level-$k$ quantal hierarchical behavior} is defined as
    $s^k_i = QBR_i(s_{-i}^{0:k-1}; \lambda, G),$
    where $s^{0:k-1}_i(a_i) \propto \sum_{m=0}^{k-1} D(m) s^m_i(a_i)$.
    The \emph{quantal cognitive hierarchy} behavior is the weighted average of these behaviors $QCH_i(a_i) = \sum_{k} D(k) s^k_i(a_i)$.
\end{definition}

QCH models depend critically on the level-0 model, which determines not just the behavior of level-0 agents, but also of higher-level agents who react to it.
A natural, simple choice is the uniform distribution, which was originally used in the Level-$k$~\citep{Nagel1995, CostaGomes2001} and Cognitive Hierarchy~\cite{Camerer2004} models.
However, \citet{JAIR} 
found improvements using richer level-0 models based on heuristics from cognitive science. 
For example, one such heuristic is the \emph{maxmax behavioral model}, 
\[
    f^{\text{maxmax}}_i(a_i) \propto \begin{cases}
        1, & a_i \in \argmax_{a_i'} \max_{a_{-i}} u_i(a_i', a_{-i}); \\ 
        0, &\text{otherwise.}
    \end{cases} 
\]
Notably, this heuristic depends only on the agent's utility $u_i$, neglecting the opponent's utility $u_{-i}$.
The other heuristics (see Appendix~\ref{apx:heuristics}) also depend on simple linear combinations of the players' utilities, such as the sum (or \emph{welfare}) $u_1 + u_2$ or the difference (\emph{unfairness}) $u_1 - u_2$.

\subsection{GameNet}
GameNet~\cite{GameNet} is a deep learning architecture for predicting human strategic behavior.
It can be understood as taking the insight that a richer level-0 specification can yield better performance to its logical limit.
Broadly, GameNet consists of two parts: \emph{feature layers}, which are intended to model level-0 behavior, and \emph{action response (AR) layers}, which perform iterative strategic reasoning.
AR layers are essentially a generalization of the QCH model, 
with additional parameters independently controlling each level's precision, the distribution over lower levels to which agents respond, and the transformed utility matrices which they use for this response.
We devote more attention to the feature layers.

Let the 0th hidden layer consist of the two matrices $H^{0,1} = U^1$ and $H^{0,2} = U^2$. For each hidden layer $0 \le \ell < L$, the matrices are first transformed by \emph{pooling units} into
\[
    P^{\ell,c} = \begin{cases}
        H^{\ell,c} & \text{if } c \le C_\ell; \\
        \rowmax(H^{\ell,c-C_\ell}) & \text{if } C_\ell < c \le 2 C_\ell; \\
        \colmax(H^{\ell,c-2C_\ell}) & \text{if } 2 C_\ell < c \le 3 C_\ell,
    \end{cases}
\]

where $\rowmax$ and $\colmax$ are functions that replace each entry of a matrix with the maximum value in its row or column, respectively; that is, for all $X \in \R^{n \times m}$, $\rowmax(X)_{ij} = \max_{1 \le a \le m} X_{ia}$ and $\colmax(X)_{ij} = \max_{1 \le a \le n} X_{aj}$.
Then, the next layer's \emph{hidden units} are
$H^{\ell,c} = \relu\left( \sum_{c'=1}^{3C_{\ell-1}} w^\ell_{c,c'} P^{\ell-1, c'} + b^\ell_c \right)$,
where $\{C_\ell\}_{0:L}$ describe the sizes of each hidden layer, including the input layer with size $C_0 = 2$, and the ReLU operation $\relu(x) = \max\{0, x\}$ is applied pointwise. 

After the final hidden layer, the matrices $\{H^{L,c}\}_{c=1:C_L}$ are transformed into a single distribution over the row player's actions as 
$f = \sum_{c=1}^{C_L} w_c f^c$,
where $f^c = \softmax(\sum_{j} H^{L,c}_{i,j})$, 
$\softmax(x)_i = \exp(x_i) / \sum_{i'} \exp(x_{i'})$, 
and the weights $w_c \in \Delta(C_L)$ are subject to simplex constraints. An analogous predicted distribution over the column player's actions is made by replacing the input utility matrices $U^1$ and $U^2$ with $(U^2)^T$ and $(U^1)^T$, respectively. 
This architecture is summarized in \Cref{fig:gamenet-diagram}.

Compared to a more standard, off-the-shelf architecture---e.g., flattening the utility matrices into a vector of length $2nm$ and applying a feedforward neural network---feature layers have several advantages.
One is that they are \emph{permutation equivariant}, which guarantees that permuting the utility matrices will permute the predictions correspondingly.
Permutation equivariance lowers the number of parameters of the network, making it learn more efficiently and removing the need for data augmentation.
Feature layers are also agnostic to the size of the game, making it possible to learn from heterogeneous data with a variety of action spaces and to generalize to games of new sizes.

\begin{figure}
    \centering
    \includegraphics[width=\columnwidth]{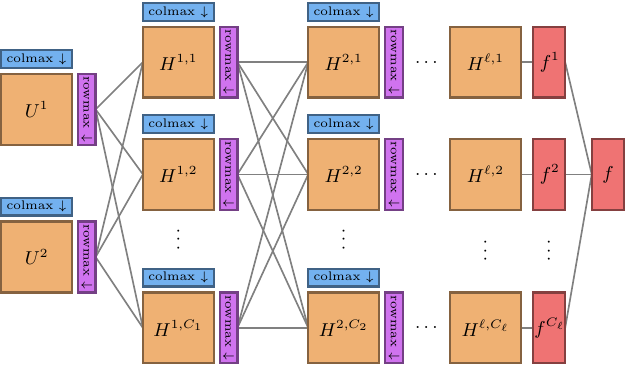}
    \caption{GameNet's feature layers.}
    \label{fig:gamenet-diagram}
\end{figure}

\subsection{Non-Strategic Behavioral Models}
\label{sec:nonstrategic}
To make precise claims about whether or not a model is strategic, we adopt a formal definition of (non-)strategic behavior from \citet{WLBnonstrategic}.
In particular, we use their definition of a \textit{weakly non-strategic model}.
\begin{definition} 
    An action $a_i^+$ in a game $G = (A, u)$ is \emph{$\zeta$-dominant} if $u_i(a_i^+, a_{-i}) > u_i(a_i, a_{-i}) + \zeta$ for all $a_i \neq a_i^+$ and $a_{-i} \in A_{-i}$. 
    Then, a behavioral model $f_i$ is \emph{dominance responsive} if there exists some $\zeta > 0$ such that, in all games $G$ with a $\zeta$-dominant action $a_i^+$, the mode of $f_i$ is $a_i^+$: that is, $f_i(G)(a_i^+) > f_i(G)(a_i)$ for all $a_i \neq a_i^+$.
\end{definition}
\begin{definition} 
    A behavioral model $f_i$ is \emph{weakly non-strategic} if it cannot be represented as quantal best response to some dominance-responsive model $f_{-i}$. 
\end{definition}

\citet{WLBnonstrategic} also define a class of non-strategic behavioral models called \textit{elementary models}.
These models independently compute a ``potential'' for each outcome in the game, then predict an action distribution based only on these potentials.
Intuitively, as long as the potential function is well-behaved, compressing two utilities into one potential creates an information bottleneck that prevents the model from representing strategic behavior. 
\begin{definition} 
    A function $\varphi: \R^2 \to \R$ is \emph{dictatorial} iff it is completely determined by one input: that is, either $\varphi(x, y) = \varphi(x, y')$ for all $x,y,y' \in \R$, or $\varphi(x,y)=\varphi(x',y)$ for all $x,x',y \in \R$.
\end{definition}
\begin{definition} 
    A function $\varphi: \R^2 \to \R$ is \emph{non-encoding} iff, for all $i \in \{1, 2\}$ and $b > 0$, there exist $x, x' \in \R^2$ such that $\varphi(x) = \varphi(x')$ but $|x_i - x'_i| > b$.
\end{definition} 
\begin{definition} 
    An \emph{elementary behavioral model} is a model of the form $f_i(G) = h_i(\Phi(G))$, where 
    \begin{enumerate}
        \item $\Phi$ maps an $n \times m$ game $G$ to a potential matrix $\Phi(G)$ by applying a potential function $\varphi: \R^2 \to \R$ to each utility vector $(u_1(a), u_2(a))$;
        \item $\varphi$ is either dictatorial or non-encoding; and
        \item $h_i$ is an arbitrary function mapping an $n \times m$ potential matrix to a vector of $n$ probabilities. 
    \end{enumerate}
\end{definition}

The key fact that we will leverage 
is that convex combinations of elementary models are weakly non-strategic.
\begin{theorem} \cite[Theorem 5,][]{WLBnonstrategic}%
    \label{thm:combination-of-elementary-nonstrategic}
    Let $g_i^1, \dots, g_i^K$ be elementary behavioral models,
    and let $w_1, \dots, w_K$ be non-negative weights summing to 1.
    Then, the convex combination $f_i(G) = \sum_{k=1}^K w_k g_i^k(G)$ is a weakly non-strategic behavioral model.
\end{theorem}

\section{GameNet's Feature Layers are Strategic}
\label{sec:gamenet-strategic}
We now have a framework with which we can formally study GameNet's feature layers.
Recall that these feature layers output a predicted level-0 behavior;
indeed, they are flexible enough to represent many existing level-0 models, such as the maxmax heuristic.
However, this flexibility turns out to be a double-edged sword.

Our first main result is that GameNet's level-0 model can represent a particular strategic model---quantal best response to maxmax---to arbitrary precision.
We give a constructive proof, providing parameter values for a 3-layer network that approximates this strategic model.

\begin{theorem} 
    \label{thm:gamenet-strategic}
    Let $q_1(G) = QBR_1(\text{maxmax}_{2}(G); 1, G)$.
    Let $\mathcal{G}$ be the set of games where all utilities are between $0$ and $C_{max}$ and all utilities differ by at least $C_{gap}$, where $C_{max}, C_{gap} > 0$ are arbitrary constants. 
    Then, there exists an instantiation of GameNet's feature layers that coincides with $q_1(G)$ for all $G \in \mathcal{G}$.
\end{theorem}

\begin{proof}
    We first give a series of computations that produce $q_1(G)$ for all $G \in \mathcal{G}$.
    Let
    \begin{align*}
        M_c &= \colmax(U^2), \\
        M_* &= \rowmax(M_c), \\
        B &= \relu(M_c / C_{gap} - M_* / C_{gap} + 1), \\
        E &= \relu(U^1 + C_{max} B - C_{max}), \text{and} \\
        Q &= \textstyle \softmax([\sum_j E_{i,j}]_{i=1}^n).
    \end{align*}
    Here, $M_c$ is a matrix where each column contains the maximum utility that player 2 could realize by playing that action.
    $M_*$ is a constant matrix containing player 2's maxmax value. 
    Then, because all utilities differ by at least $C_{gap}$, $B$ is a matrix containing a column of ones for player 2's maxmax action and zeros in all other columns. (This maxmax action is unique because all of the utilities are distinct.)
    Because all utilities are between $0$ and $C_{max}$, $E$ is a matrix containing player 1's utilities in player 2's maxmax column and zeros elsewhere.
    Finally, $Q$ is a vector containing player 1's quantal best response to the maxmax action.
    Therefore, $Q = q_i(G)$.

    We now show that GameNet's feature layers can represent these computations.
    Consider a model with three hidden layers, with two hidden units in the first two layers and one hidden unit in the final layer.
    Assume that all unspecified weights and biases are set to zero.
    In the first layer, setting $w^1_{1,2} = 1$ gives $H^{1,1} = U^1$, 
    and setting $w^1_{2,6} = 1$ gives $H^{1,2} = M_c$.
    In the second layer, setting $w^2_{1,1} = 1$ gives $H^{2,1} = U^1$, and  setting $w^2_{2,2} = 1 / C_{gap}$, $w^2_{2,4} = -1 / C_{gap}$, and $b^2_2 = 1$ gives $H^{2,2} = B$. 
    In the last layer, setting $w^3_{1,1} = 1$, $w^3_{1,2} = C_{max}$, and $b^3_1 = -C_{max}$ gives $H^{3,1} = E$.
    The softmax operations at the end of the feature layers complete the proof.
\end{proof}

Note that the assumption that the games in $\mathcal{G}$ have positive utilities was made purely for clarity of exposition.
It is straightforward to handle negative utilities by adding appropriate constant shifts.
We prove this more general claim in Appendix~\ref{apx:gamenet-negative}

\section{A Non-Strategic Neural Network}
\label{sec:enet}
Having seen that GameNet's feature layers can represent strategic reasoning, 
we now introduce a new neural network architecture that is only capable of non-strategic behavior.

\subsection{ElementaryNet}
In \Cref{thm:combination-of-elementary-nonstrategic}, we saw that it is possible to construct non-strategic behavioral models by composing an arbitrary response function with a potential function, as long as this potential function is either dictatorial or non-encoding.
Intuitively, as long as the potential function is nicely behaved, it adds an information bottleneck to the model, discarding information about the agents' utilities before the response function can perform more complex computations.

This inspires ElementaryNet, a new architecture that adds such an information bottleneck to GameNet's flexible feature layers.
ElementaryNet is a model of the form
$f_i(G) = \sum_{p=1}^P w_p \cdot h^p_i(\Phi^p(G))$,
where $\Phi^p: \R^2 \to \R$ are parameterized potential functions applied elementwise to the utility matrices; $h^p$ are parameterized response functions analogous to GameNet's feature layers; and $w_p$ is a vector of probabilities.
\Cref{fig:enet-diagram} illustrates this architecture.

We will study two particular instantiations of this model, differing in the specification of their potential functions.
The first instantiation, which we dub the \emph{learned-potential} model, allows the potentials to be arbitrary linear functions
\[
    \varphi^p(x,y) = \theta^p_x x + \theta^p_y y,
\]
where the coefficients $\theta^p_x$ and $\theta^p_y$ are trainable parameters.
This model is therefore able to use any linear potential function, so long as it is justified by the training data.

The second instantiation, which we dub the \emph{fixed-potential} model, uses four fixed linear potential functions:
\begin{align*}
    \varphi_{\text{own}}(x,y) &= x, \\
    \varphi_{\text{opp}}(x,y) &= y, \\
    \varphi_{\text{sum}}(x,y) &= x+y, \text{and} \\
    \varphi_{\text{diff}}(x,y) &= x-y.
\end{align*}
This set of potential functions is natural.
Noticing that the sign of the potentials is unimportant (as the response function can easily negate all potential values before applying any other computation), these are the four distinct linear potential functions that can be constructed using the coefficients $-1$, $0$, and $1$.
All four also have an economic interpretation.
$\varphi_{\text{own}}$ describes ``single-agent'' reasoning about the agent's own payoffs;
$\varphi_{\text{opp}}$ describes purely altruistic reasoning about the opponent's payoffs;
$\varphi_{\text{sum}}$ computes the welfare of each outcome; and
$\varphi_{\text{diff}}$ measures the (un)fairness of each outcome.
Accordingly, these potential functions---specifically, $\varphi_{\text{own}}$, $\varphi_{\text{sum}}$, and $\varphi_{\text{diff}}$---can be used to represent all of the level-0 heuristics from \citet{JAIR}.
A fixed-potential ElementaryNet model has no trainable parameters in its potential functions; its parameters consist solely of those in the response functions $h^p$ and the convex combination $w_p$.

\begin{figure}
    \centering
    \includegraphics[width=\columnwidth]{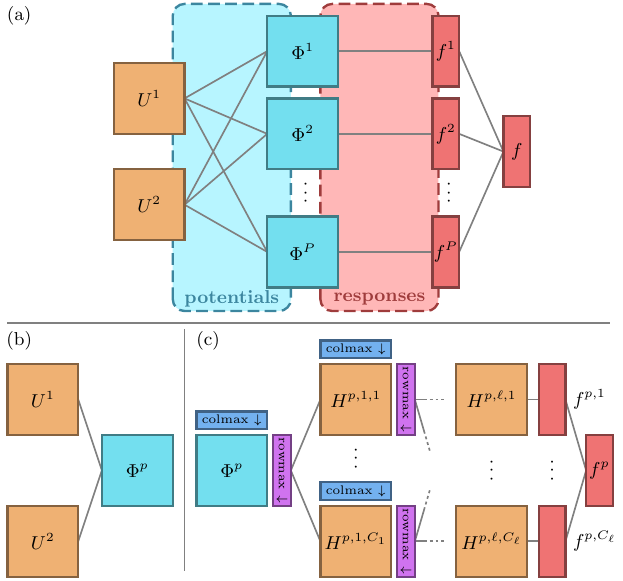}
    \caption{The ElementaryNet architecture. (a) The full model; (b) a potential function; and (c) a response function.}
    \label{fig:enet-diagram}
\end{figure}

\subsection{ElementaryNet is Non-Strategic}
Our second theoretical result is that ElementaryNet is weakly non-strategic: 
it cannot represent quantal best response to any dominance-responsive behavioral model.

\begin{theorem}
    ElementaryNet is weakly non-strategic.
\end{theorem}

\begin{proof}
    We will first show that any linear potential function $\varphi(x, y) = \theta_x x + \theta_y y$ is either dictatorial or non-encoding.
    If $\theta_x = 0$ or $\theta_y = 0$, then $\varphi$ is dictatorial. 
    Otherwise, let $b > 0$ be arbitrary.
    Let $(x_1, y_1) = (0, 0)$ and $(x_2, y_2) = (2b, -2b \frac{\theta_x}{\theta_y})$. Then, we have $\varphi(x_1, y_1) = \varphi(x_2, y_2) = 0$ but $|x_1 - x_2| = 2b > b$.
    Similarly, let $(x_3, y_3) = (2b \frac{\theta_y}{\theta_x}, -2b)$; 
    we have $\varphi(x_1, y_1) = \varphi(x_3, y_3) = 0$ but $|y_1 - y_3| = 2b > b$.
    Therefore, $\varphi$ is non-encoding. 

    Because each of its potential functions are either dictatorial or non-encoding, ElementaryNet is therefore a convex combination of elementary models.
    Then, \Cref{thm:combination-of-elementary-nonstrategic} implies that it is weakly non-strategic, completing the proof.
\end{proof}

Intuitively, linear functions have linear level curves, which are either axis-aligned (implying that the function is dictatorial) or extend arbitrarily far in both dimensions (implying that the function is non-encoding).
In either case, linear potential functions discard enough information to ensure that the model cannot represent strategic reasoning.
The flexibility of the response functions is immaterial: it is the bottleneck in the potential functions that preclude strategic behavior.

Notably, ElementaryNet does not just disagree with quantal best response on isolated games.
Instead, it is incapable of representing broad categories of strategic reasoning.
It can fail to be ``other-responsive'', disregarding the other player's preferences entirely; such a model cannot possibly form accurate beliefs about an opponent. 
Otherwise, if it is other-responsive, it can then be made to play dominated actions with probability bounded away from zero, regardless of how large the losses in utility are.
We formalize and prove these properties in Appendix~\ref{apx:enet-nonstrategic}.

\section{Experiments}
\label{sec:experiments}

While the constraints in ElementaryNet's architecture make it provably unable to represent strategic reasoning, 
only an empirical study can determine the extent to which these constraints affect the model's ability to model human behavior.
Here we show the most positive result that we could hope for: that, when ElementaryNet is used as the level-0 specification for an iterative reasoning model, it matches the performance of GameNet.
What's more, since our model cleanly factors descriptions of strategic and nonstrategic behavior into different parameters, we can gain insights into human behavior by varying features of the model and analyzing its parameters.
We demonstrate how this type of analysis can yield interpretable results,
showing that iterative reasoning is a good model of human behavior; 
that the precise model of iterative reasoning is relatively unimportant; 
that far simpler level-0 specifications are inconsistent with non-strategic behavior in our data;
and that level-0 specifications from prior work lose relatively little predictive power.

\subsection{Experimental Setup}

\paragraph{Data.}
We used a dataset consisting of results from twelve experimental studies.
The first ten were used in past work comprehensively evaluating behavioral game theory models \cite{WLB2017, JAIR}.
The remaining two studies ran large-scale experiments on Amazon Mechanical Turk \cite{FL19, CHW23}.
In total, the dataset contains 26,553 observations across 366 distinct games.  
We provide more details about this data in Appendix~\ref{apx:experiments}.

\paragraph{Training.}
To train our models, we first randomly selected 20\% of the games to use as a validation set and 20\% to use as a test set, using the remaining 60\% as a training set.
Then, we trained up to 36 models with different hyperparameters, varying the L1 regularization coefficient applied to the neural network weights, the dropout probability, and the initial QCH model parameters.
The exact hyperparameter values we used are detailed in Appendix~\ref{apx:experiments}. 
We report the test loss of the model that had the lowest validation loss.
We used the squared L2 error between the predicted distribution and empirical distribution as our loss function, as past work has argued that it is appropriate for evaluating behavioral models~\cite{dEon2024}.
We repeated this procedure for 50 train/validation/test splits.

\paragraph{Confidence intervals.}
Due to the small size of our dataset and the comparative flexibility of our models, 
the losses depend heavily on precisely which games are in the training, validation, and test sets, adding substantial variance to our training procedure.
To combat this high variance, rather than reporting losses of our models, we report the difference in loss to a reference model on the same data split.
Taking paired differences in this way removes any variance caused purely by differences in the data, isolating differences in model performance.
We report bias-corrected accelerated bootstrapped confidence intervals~\cite{Efron1987} of the differences in test loss with confidence levels of both 68\% and 95\%.
We also report absolute model losses in Appendix~\ref{apx:results}.

\subsection{Comparing to GameNet}
First, to evaluate whether ElementaryNet's restrictions harmed its predictive performance, we compared it to GameNet.
We trained QCHp models using a variety of GameNet and ElementaryNet models as their level-0 input.
The ElementaryNet models each used a single learned potential function.
We varied GameNet's feature layers and ElementaryNet's response functions to use 1, 2, or 3 hidden layers with widths of 10, 20, or 50.
We also evaluated the Uniform + QCHp model as a baseline.

Our results are shown in \Cref{fig:results-performance}.
Consistent with the results reported by \citet{GameNet}, we found that the best GameNet models performed far better than the Uniform + QCHp baseline.
GameNet models with only a single hidden layer performed best, while deeper models had worse test loss, a likely sign of overfitting.
However, the best ElementaryNet models---which also had one hidden layer---had nearly identical performance, with test losses that are statistically indistinguishable from the best GameNet model.
These results show that, despite the additional constraints on ElementaryNet, it constitutes a state-of-the-art model of human strategic behavior when used in tandem with a QCHp model.

\begin{figure}
    \centering
    \includegraphics[width=\columnwidth]{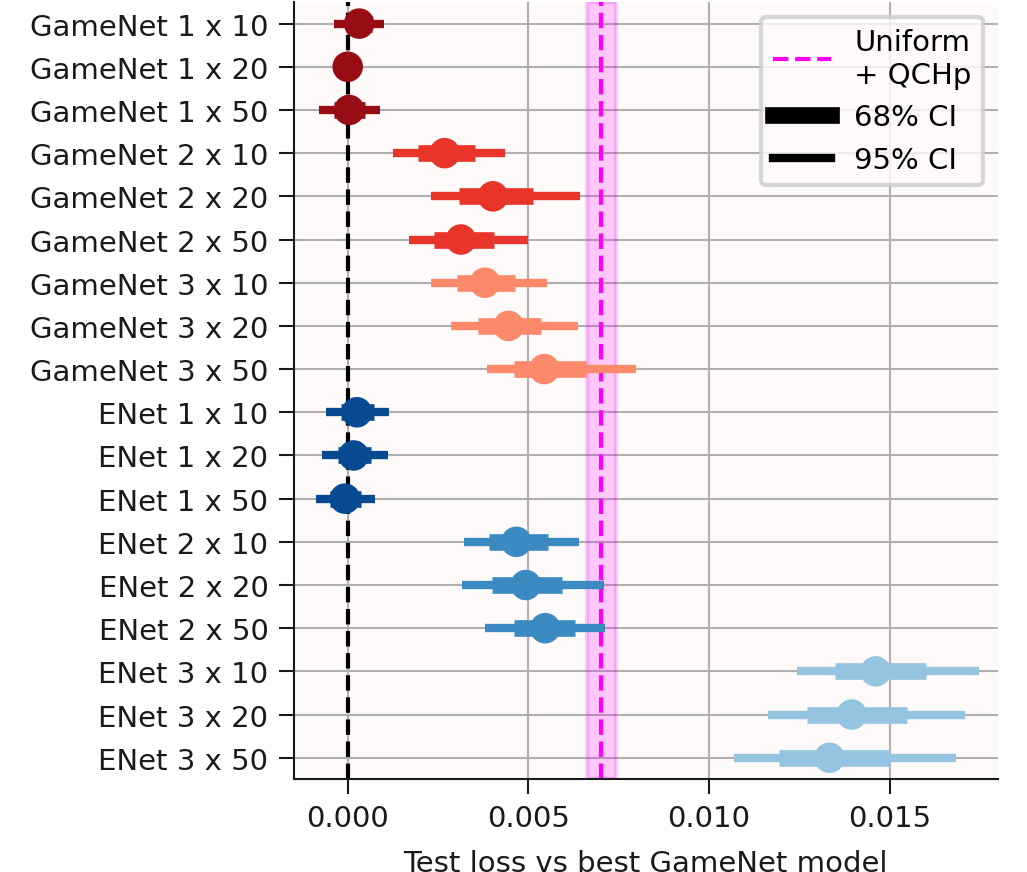}
    \caption{ElementaryNet and GameNet level-0 models trained with a QCH-Poisson strategic model. The best ElementaryNet models are similar in performance to the best GameNet model. Lower values are better.}
    \label{fig:results-performance}
\end{figure}

\subsection{Leveraging our Interpretable Model}
Having showed that ElementaryNet + QCHp is a state-of-the-art model of strategic behavior, we then sought to leverage its theoretical guarantees to understand what types of reasoning might be on display in our data.
We ran four follow-up experiments varying different features of the model. 
In each case, we found insights about the kinds of strategic or non-strategic reasoning exhibited by the participants in our dataset.

\subsubsection{Models with no strategic reasoning.}
We first compared the performance of our best ElementaryNet + QCHp model to one with no QCH model at all.
This experiment is similar to the one that \citet{GameNet} ran, where they found that GameNet performed better with only a level-0 model. 
However, we showed that GameNet's level-0 model is able to emulate strategic models, making it impossible to conclude that iterative strategic reasoning is a poor model of human strategic behavior from this empirical result.
In contrast, we can be sure that ElementaryNet's level-0 model alone is non-strategic.

We trained learned-potential ElementaryNet models with 1, 2, and 3 hidden layers, with 50 hidden units in each layer, and no strategic model.
The results (\Cref{fig:results-ablation-noqch}) show that these purely non-strategic models fit the data extremely poorly, with test losses significantly worse than even the baseline model.
Thus, we can conclude that iterative reasoning is indeed a good model of the behavior of our subjects, far outstripping any non-strategic model.

\begin{figure}
    \centering
    \includegraphics[width=\columnwidth]{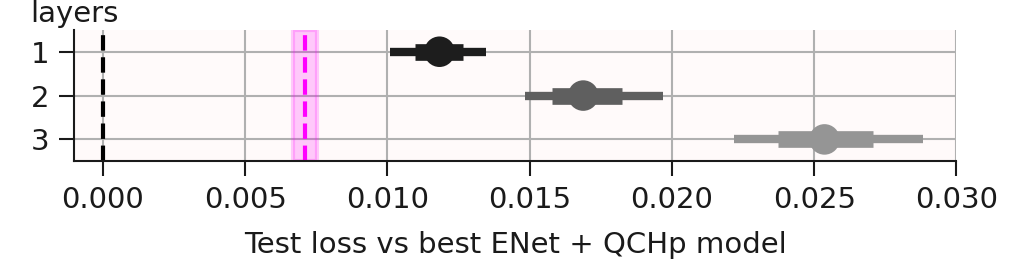}
    \caption{ElementaryNet models with no QCH model. Purely non-strategic models fit the data extremely poorly.}
    \label{fig:results-ablation-noqch}
\end{figure}   

\subsubsection{Different levels of strategic reasoning.}
Next, we varied the strategic model.
We focused on our best performing level-0 model: ElementaryNet with a learned potential and a response function consisting of 1 hidden layer of 50 units.
However, we replaced the strategic model with QCH models having arbitrary level distributions up to 1, 2, or 3 levels, allowing these level distributions to be learned during training.
The results (\Cref{fig:results-ablation-qch}) show that the differences between the various QCH models were relatively small.
ElementaryNet models trained with QCH1 and QCH3 had performance that was statistically indistinguishable from the results with the original model using QCHp.
The model using QCH2 performed the worst, but still had an average test loss far better than that of the baseline Uniform + QCHp model.

\begin{figure}
    \centering
    \includegraphics[width=\columnwidth]{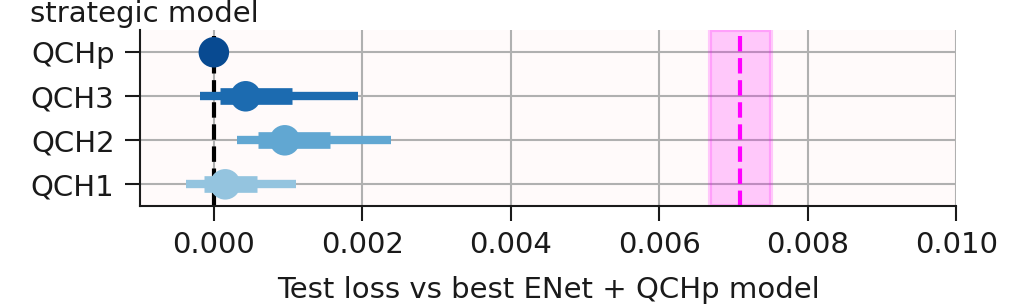}
    \caption{ElementaryNet level-0 models with various QCH models. Changing the structure of the level distribution has little effect on model performance.}
    \label{fig:results-ablation-qch}
\end{figure}   

Each of these QCH models have interpretable level distributions, which are either discrete histograms (for the QCH1, QCH2, and QCH3 models) or Poisson distributions (for the QCHp model). 
\Cref{fig:results-parameters-qch} shows the average learned level distribution for each of the ElementaryNet models, as well as the baseline Uniform + QCHp model.
These plots show that each of the QCH models co-trained with an ElementaryNet level-0 model place over 70\% of their probability on level-0 reasoners, while the baseline model places only 33\% of its probability on level-0 reasoners. 
This difference suggests that the majority of subjects in our data are likely performing some type of relatively rich non-strategic reasoning, which the baseline model is forced to model as imperfect strategic behavior.
This qualitative finding echoes previous results from \citet{JAIR}, who found that using level-0 models beyond the uniform baseline increased the fitted proportion of level-0 reasoners from 30\% to 60\%.

\begin{figure}
    \centering
    \includegraphics[width=\columnwidth]{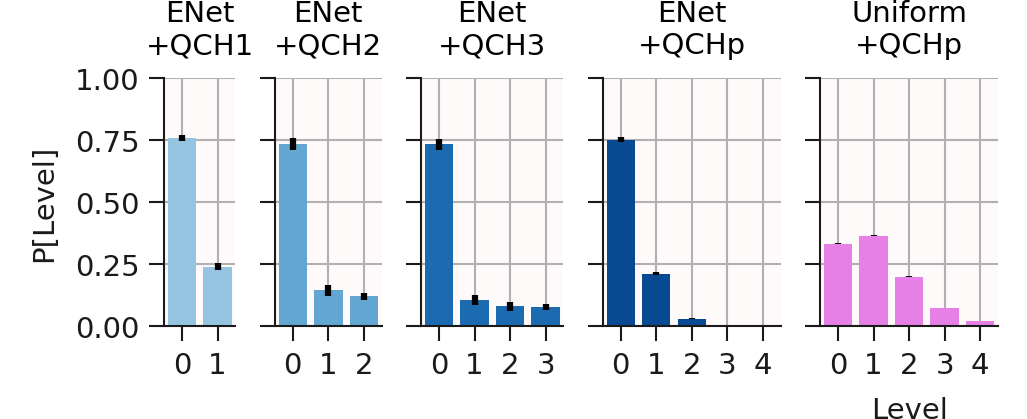}
    \caption{Fitted QCH parameters. Error bars show 95\% confidence intervals on parameter values.}
    \label{fig:results-parameters-qch}
\end{figure}

\subsubsection{Simpler level-0 specifications.}
Having studied our strategic model extensively, we next varied features of our level-0 model.
A natural question to ask here is whether our nuanced definition of non-strategic behavior was necessary.
After all, it would be easier to simply create a level-0 model that makes predictions based only on the agent's own payoffs.
Such a model is clearly non-strategic, as it cannot reasonably form beliefs about the opponent's behavior without knowing the opponent's payoffs.
Would this simpler level-0 specification be sufficient for modelling our subjects' behavior?

We trained an ElementaryNet model with a single potential function fixed to the ``own'' potential, which conforms to this more restrictive level-0 specification.
As before, we used a response function with one hidden layer of 50 units and a QCHp strategic model.
The results (\Cref{fig:results-ablation-potentials}, light green) show that this simpler level-0 model performed far worse than our best ElementaryNet model, with average test loss closer to that of the Uniform + QCHp baseline.
This provides evidence that the experimental subjects exhibited rich non-strategic behavior that cannot be expressed without knowledge of the opponent's payoffs.

\subsubsection{Level-0 potentials from prior work.}
Lastly, we considered a middle ground: an ElementaryNet model with multiple fixed potential functions.
We trained an ElementaryNet model with four potential functions, using each of the four fixed potential functions that we described in \Cref{sec:enet}.
This model is unable to adapt its potential functions to the data, but can still express richer non-strategic behavior, such as the heuristics used by~\citet{JAIR}.
As before, the model still used response functions with one hidden layer of 50 units and a QCHp strategic model.
The results (\Cref{fig:results-ablation-potentials}, dark green) show that the model with four fixed potentials far outperformed the model with just one.
However, it still did not match the performance of our best ElementaryNet model, which had just a single learned potential function.
This result suggests that it may be possible to develop simple level-0 heuristic models that outperform those from past work by considering more nuanced potential functions beyond welfare and fairness.

\begin{figure}
  \centering
  \includegraphics[width=\columnwidth]{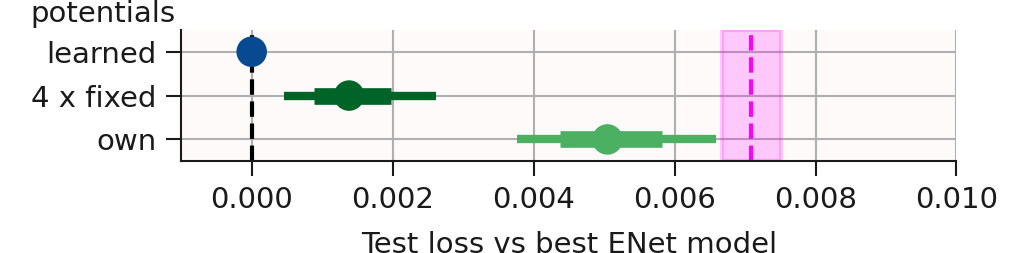}
  \caption{ElementaryNet level-0 models with various sets of potential functions. Both alternative specifications of non-strategic behavior degraded performance.} 
  \label{fig:results-ablation-potentials}
\end{figure}

\section{Conclusions and Discussion}
Model performance often comes at the cost of interpretability.
In this work, we proved that GameNet, an opaque, state-of-the-art predictor of human strategic behavior, has deep-seated interpretability problems: its purportedly level-0 model is able to emulate strategic reasoning.
However, we showed that ElementaryNet, a variation of GameNet with additional restrictions, is completely unable to represent strategic behavior.
These restrictions make the model far easier to interpret, making it possible to derive insights about human behavior by inspecting the model's learned parameters or varying its features, without losing predictive power. 

We see several promising directions for future work; we describe three here.
First, we focused on interpreting the strategic model and the potential functions, two pieces of the model with interpretable parameters, but not the response functions, which are still black boxes.
Future work could analyze these response functions by studying how they behave in simple cases, or by testing whether they continue to work well with additional restrictions.

Second, \citet{Zhu2024} presented an alternative method for improving behavioral game theory models, using a neural network to control the parameters of a quantal response model.
Fusing their strategic models with ElementaryNet's level-0 predictions could produce interpretable, yet highly predictive, models of human behavior. 

Third, we focused on unrepeated, two-player normal-form games. 
These relatively simple games already evoke strategic behavior, but they do not include other elements such as sequential interactions, random chance, and imperfect information that are common in real-world settings.
Extending our models to handle these elements is non-trivial, as it will require developing compelling definitions of non-strategic behavior beyond normal-form games and significantly modifying our architecture.
Nonetheless, we are hopeful that the core idea of combining a rich model of non-strategic behavior with a structured model of strategic reasoning will be successful in these more complex settings.

\section*{Acknowledgments}
This work was funded by an NSERC Discovery Grant, a CIFAR Canada AI Research Chair (Alberta Machine Intelligence Institute), and computational resources provided both by UBC Advanced Research Computing and a Digital Research Alliance of Canada RAC Allocation. Additionally, Greg d'Eon and Hala Murad were supported in part by funding from the UBC Advanced Machine Learning -- Training Network, and Hala Murad was supported in part by funding from a UBC Work Learn International Undergraduate Research Award.

\bibliography{references}
\nocite{*}

\clearpage
\appendix
\appendixpage

In these appendices, we include the following additional details:
\begin{itemize}
    \item Appendix~\ref{apx:heuristics}: formal definitions of additional level-0 heuristic strategies.
    \item Appendix~\ref{apx:gamenet-negative}: extending \Cref{thm:gamenet-strategic} to handle negative utilities.
    \item Appendix~\ref{apx:enet-nonstrategic}: proving that ElementaryNet cannot represent broad categories of strategic reasoning.
    \item Appendix~\ref{apx:experiments}: additional details about experiments.
    \item Appendix~\ref{apx:results}: results in terms of absolute model performance.
\end{itemize}

\section{Level-0 Heuristics}
\label{apx:heuristics}
In this section, we enumerate the level-0 heuristic strategies used by \citet{JAIR}.

\paragraph{Maxmax.}
An action is a \emph{maxmax} action for agent $i$ if it maximizes their best-case utility.
The \emph{maxmax behavioral model} uniformly randomizes over an agent's maxmax actions: 
\[
    f^{\text{maxmax}}_i(a_i) \propto \begin{cases}
        1, & a_i \in \argmax_{a_i'} \max_{a_{-i}} u_i(a_i', a_{-i}); \\ 
        0, &\text{otherwise.}
    \end{cases} 
\]

\paragraph{Maxmin.}
An action is a \emph{maxmin} action for agent $i$ if it maximizes their worst-case utility.
The \emph{maxmin behavioral model} uniformly randomizes over an agent's maxmin actions:
\[
    f^{\text{maxmin}}_i(a_i) \propto \begin{cases}
        1, & a_i \in \argmax_{a_i'} \min_{a_{-i}} u_i(a_i', a_{-i}); \\
        0, & \text{otherwise.}
    \end{cases}
\]

\paragraph{Minimax regret.}
In an action profile $(a_i, a_{-i})$, agent $i$'s \emph{regret} is the maximum amount of utility they could gain by deviating to another action:
\[
    r(a_i, a_{-i}) = \max_{a_i'} u_i(a_i', a_{-i}) - u_i(a_i, a_{-i}).
\]
The \emph{minimax regret behavioral model} uniformly randomizes over actions that minimize the agent's worst-case regret:
\[
    f^{\text{regret}}_i(a_i) \propto \begin{cases}
        1, & a_i \in \argmin_{a_i} \max_{a_{-i}} r(a_i, a_{-i}); \\
        0, & \text{otherwise.}
    \end{cases}
\]

\paragraph{Max symmetric.} 
A game is \emph{symmetric} if the players' roles are interchangeable: that is, if the action spaces $A_1 = A_2$ are identical, and the utility functions $U_1 = U_2^T$ are unchanged by switching the identities of the players.
In a symmetric game, an action is a \emph{max symmetric} action for agent $i$ if it maximizes their utility under the assumption that the opponent plays the same action.
The \emph{max symmetric behavioral model} uniformly randomizes over max symmetric actions: for all symmetric games,
\[
    f^{\text{symmetric}}_i(a_i) \propto \begin{cases}
        1, & a_i \in \argmax_{a_i} u_i(a_i, a_i); \\
        0, & \text{otherwise.}
    \end{cases}
\]
This model is not defined for games that are not symmetric; it can be taken to output the uniform distribution on those games.

\paragraph{Maxmax fairness.}
An action profile is a \emph{maxmax fairness} action profile if it minimizes the absolute difference between the players' utilities $d(a_i, a_{-i}) = |u_1(a_i, a_{-i}) - u_2(a_i, a_{-i})|$.
The \emph{maxmax fairness behavioral model} uniformly randomizes over actions that appear in a maxmax fairness action profile:
\[
    f^{\text{fairness}}_i(a_i) \propto \begin{cases}
        1, & a_i \in \argmin_{a_i} \min_{a_{-i}} d(a_i, a_{-i}); \\
        0, & \text{otherwise.}
    \end{cases}
\]

\paragraph{Maxmax welfare.}
An action profile is a \emph{maxmax welfare} action profile if it maximizes the sum of the players' utilities $w(a_i, a_{-i}) = u_1(a_i, a_{-i}) + u_2(a_i, a_{-i})$.
The \emph{maxmax welfare behavioral model} uniformly randomizes over actions that appear in a maxmax welfare action profile:
\[
    f^{\text{welfare}}_i(a_i) \propto \begin{cases}
        1, & a_i \in \argmax_{a_i} \max_{a_{-i}} w(a_i, a_{-i}); \\
        0, & \text{otherwise.}
    \end{cases}
\]

Each of these heuristics each operate on different transformations of the utility functions.
The maxmax, maxmin, minmax regret, and max symmetric heuristics are all functions of the agent's utility $u_i$, neglecting the opponent's utility $u_{-i}$. 
Similarly, maxmax fairness is a function of the difference between the utilities $u_1(a) - u_2(a)$, and maxmax welfare is a function of the sum $u_1(a) + u_2(a)$.

\clearpage
\section{Additional Proofs: GameNet with Negative Utilities}
\label{apx:gamenet-negative}

In Section 3, we argued that it is possible to extend the proof of Theorem 3.1 to handle negative utilities. 
We formalize this claim here.
First, we restate the theorem, allowing for negative utilities.

\begin{theorem} 
    Let $q_1(G) = QBR_1(\text{maxmax}_{2}(G); 1, G)$.
    Let $\mathcal{G'}$ be the set of games where all utilities are between $-C_{max}$ and $C_{max}$ and all utilities differ by at least $C_{gap}$, where $C_{max}, C_{gap} > 0$ are arbitrary constants. 
    Then, there exists an instantiation of GameNet's feature layers that coincides with $q_1(G)$ for all $G \in \mathcal{G'}$.
\end{theorem}

The proof is similar to the proof of Theorem 3.1.

\begin{proof}
    Like in the original proof, we first give a series of computations that produce $q_1(G)$ for all $G \in \mathcal{G'}$.
    Let
    \begin{align*}
        M_c' &= \colmax(U^2 + C_{max}), \\
        M_*' &= \rowmax(M_c'), \\
        B' &= \relu(M_c' / C_{gap} - M_*' / C_{gap} + 1), \\
        E' &= \relu(U^1 + 2C_{max} B' - C_{max}), \text{and} \\
        Q' &= \textstyle \softmax([\sum_j E_{i,j}']_{i=1}^n).
    \end{align*}
    Similar to before, $M_c'$ is a matrix that contains the maximum utility that player 2 could realize by playing that action, plus $C_{max}$.
    $M_*'$ is a constant matrix containing player 2's maxmax value plus $C_{max}$. 
    Because adding $C_{max}$ to all utilities does not change the maxmax action, $B'$ is a matrix containing a column of ones for player 2's maxmax action and zeros in all other columns.
    Then, consider the term $U^1 + 2C_{max} B' - C_{max}$.
    In player 2's maxmax column, this matrix contains player 1's utilities plus $C_{max}$; because their utilities are greater than $-C_{max}$, these entries are all positive.
    In the other columns, this matrix contains player 1's utilities minus $C_{max}$; because their utilities are less than $C_{max}$, these entries are all negative.
    Thus, $E'$ contains player 1's utilities plus $C_{max}$ in player 2's maxmax column, and zeros elsewhere.
    Finally, because constant offsets do not affect softmax operations, $Q'$ is a vector containing player 1's quantal best response to the maxmax action.
    Therefore, $Q' = q_i(G)$.

    Then, consider a GameNet model as before with three hidden layers, with two hidden units in the first two layers and one hidden unit in the final layer.
    In the first layer, setting $w^1_{1,2} = 1$ and $b^1_1 = C_{max}$ gives $H^{1,1} = U^1 + C_{max}$, 
    and setting $w^1_{2,6} = 1$ and $b^1_2 = C_{max}$ gives $H^{1,2} = M_c' + C_{max}$.
    In the second layer, setting $w^2_{1,1} = 1$ gives $H^{2,1} = U^1 + C_{max}$, and setting $w^2_{2,2} = 1 / C_{gap}$, $w^2_{2,4} = -1 / C_{gap}$, and $b^2_2 = 1$ gives $H^{2,2} = B'$. 
    In the last layer, setting $w^3_{1,1} = 1$, $w^3_{1,2} = 2C_{max}$, and $b^3_1 = -2C_{max}$ gives $H^{3,1} = E'$.
    The softmax operations at the end of the feature layers complete the proof.
\end{proof}

\clearpage
\section{Additional Proofs: ElementaryNet Cannot Approximate Strategic Behavior}
\label{apx:enet-nonstrategic}

In Section 4, we argued that ElementaryNet does not just disagree with quantal best response models on isolated games, but is instead unable to represent entire broad categories of strategic reasoning. 
Here, we formalize this claim.

First, we describe two properties of quantal best response. 
The first is that it does not just place its modal probability on dominant actions: as the margin of dominance grows, it plays them with probability approaching $1$.
\begin{definition} 
    Let $f_i$ be a behavioral model. 
    For any game $G$ containing a dominant action $a_i^+$, let $D(f_i, G) = f_i(G)(a_i^+)$ be the probability that $f_i$ places on the dominant action.
    Let $\mathcal{G}(A, \zeta)$ be the set of games with action space $A$ containing a $\zeta$-dominant action.
    Let $M(f_i, A, \zeta) = \min_{G \in \mathcal{G}(A, \zeta)} D(f_i, G)$ be the minimum probability that $f_i$ places on a $\zeta$-dominant action over all games with action space $A$.
    We say that $f_i$ satisfies \emph{strong dominance responsiveness in the limit} iff $\lim_{\zeta \to \infty} M(f_i, A, \zeta) = 1$ for all action spaces $A$.
\end{definition}
In other words, if a model satisfies strong dominance responsiveness in the limit, it may sometimes play dominated actions, but it can be made to play these actions arbitrarily rarely by increasing the margin by which they are dominated.
\begin{proposition}
    For any positive precision $\lambda > 0$ and behavioral model $f_{-i}$, the quantal best response $q_i(G) = QBR_i(f_{-i}(G); \lambda, G)$ satisfies strong dominance responsiveness in the limit.
\end{proposition}
\begin{proof}
    Let $G \in \mathcal{G}(A, \zeta)$ be an arbitrary game with a $\zeta$-dominant action $a_i^+$, and let $s_{-i} = f_{-i}(G)$.
    First, for any action $a_i \neq a_i^+$, the difference between the expected utilities of these two actions satisfies
    \begin{align*}
        &u(a_i^+, s_{-i}) - u(a_i, s_{-i}) \\
        &=   \E_{a_{-i} \sim s_{-i}} u(a_i^+, a_{-i}) - u(a_i, a_{-i}) \\
        &\ge \E_{a_{-i} \sim s_{-i}} \zeta \\
        &= \zeta,
    \end{align*}
    where the inequality follows from the fact that $a_i^+$ is $\zeta$-dominant.
    Then, the probability that $q_i$ places on the dominant action is
    \begin{align*}
        &q_i(G)(a_i^+) \\
        &= \frac{\exp(\lambda \cdot u_i(a_i^+, s_{-i}))}{\sum_{a_i}\exp(\lambda \cdot u_i(a_i, s_{-i}))} \\
        &= \frac{\exp(\lambda \cdot u_i(a_i^+, s_{-i}))}{\exp(\lambda \cdot u_i(a_i^+, s_{-i})) + \sum_{a_i \neq a_i^+} \exp(\lambda \cdot u_i(a_i, s_{-i}))} \\
        &= \frac{1}{1 + \sum_{a_i \neq a_i^+} \exp(\lambda \cdot (u_i(a_i, s_{-i}) - u_i(a_i^+, s_{-i})))} \\
        &\ge \frac{1}{1 + \sum_{a_i \neq a_i^+} \exp(-\lambda \zeta)} \\
        &= \frac{1}{1 + (|A_i| - 1) \exp(-\lambda \zeta).}
    \end{align*}
    As $\zeta \to \infty$, this lower bound converges to $1$, completing the proof.
\end{proof}

The second property of quantal best response is that---under mild conditions on the opponent's strategy---it is also sensitive to changes in the \emph{opponent's} payoffs.
\begin{definition}
    A behavioral model $f_i$ is \emph{other-responsive} iff there are a pair of games $G = (A, u)$ and $G' = (A, u')$ such that $u_i(a) = u'_i(a)$ for all $a \in A$, but $f_i(G) \neq f_i(G')$.
\end{definition}
\begin{proposition}
    For any positive precision $\lambda > 0$ and dominance-responsive behavioral model $f_{-i}$, the quantal best response $q_i(G) = QBR_i(f_{-i}(G); \lambda, G)$ satisfies other-responsiveness.
\end{proposition}
\begin{proof}
    Consider the following two games $G$ and $G'$:
    \begin{center}
        \begin{tabular}{cc}
        \begin{tabular}{r|ll}
            \toprule
              & L      & R      \\
            \midrule
            U & (1, 1) & (0, 0) \\
            D & (0, 1) & (1, 0) \\
            \bottomrule
        \end{tabular} &
        \begin{tabular}{r|ll}
            \toprule
              & L      & R      \\
            \midrule
            U & (1, 0) & (0, 1) \\
            D & (0, 0) & (1, 1) \\
            \bottomrule
        \end{tabular} 
        \vspace{0.5em} \\
        G & G'
        \end{tabular}
    \end{center}
    Let $i$ be the row player.
    We have $u_i(a) = u'_i(a)$ for all $a \in A$.
    It remains to be shown that $q_i(G) \neq q_i(G')$.
    
    Consider the first game $G$.
    Let $s_{-i} = f_{-i}(G)$; because $f_{-i}$ is dominance responsive and L is a dominant action for the column player, we have $s_{-i}(L) > 0.5 > s_{-i}(R)$.
    Then, for the row player, we have $u_i(U, s_{-i}) > 0.5 > u_i(D, s_{-i})$.
    This implies that
    \begin{align*}
        & q_i(G)(U) \\
        &= \frac{\exp(\lambda u_i(U, s_{-i}))}{\exp(\lambda u_i(U, s_{-i})) + \exp(\lambda u_i(D, s_{-i}))} \\
        &> \frac{\exp(\lambda u_i(U, s_{-i}))}{2\exp(\lambda u_i(U, s_{-i}))} \\
        &= \frac{1}{2}.
    \end{align*}
    In the second game $G'$, R is a dominant action for the column player; a similar argument shows that $q_i(G')(U) < \frac{1}{2}$.
    Therefore, $q_i(G) \neq q_i(G')$, completing the proof.
\end{proof}

In their proof of Theorem 2.8, \citet{WLBnonstrategic} proved a stronger result that any convex combination of elementary models can be sensitive to their inputs in one of these two ways, but not both.
\begin{lemma}
    \label{lemma:combination-of-elementary-nonstrategic}
    Let $f_i(G) = \sum_{k=1}^K w_k g_i^k(G)$ be a convex combination of elementary behavioral models.
    Then, either $f_i$ is not strongly dominance-responsive in the limit, or it is not other responsive.
\end{lemma}
\begin{proof}
    Consider the potential functions $\varphi^k$ in each of the component models $g_i^k$ having weight $w_k > 0$.
    If all of these potentials are dictated by player $i$'s payoff, then the potential outputs are constant with respect to the opponent's payoffs;
    in this case, the output probabilities must also be constant with respect to the opponent's payoffs, meaning that the model is not other-responsive.
    
    Otherwise, there exists some $k$ with positive weight $w_k > 0$ such that the potential is not dictated by player $i$'s payoff.
    In this case, for all $\zeta > 0$, there exist two payoff tuples $x = (x_i, x_{-i})$ and $x' = (x'_i, x'_{-i})$ such that $x_i - x_{-i} > \zeta$ but $\varphi(x) = \varphi(x')$.
    Then, consider the following games:
    \begin{center}
    \begin{tabular}{cc}
        \begin{tabular}{r|ll}
            \toprule
              & L    & R    \\
            \midrule
            U & $x$  & $x$  \\
            D & $x'$ & $x'$ \\
            \bottomrule
        \end{tabular} &
        \begin{tabular}{r|ll}
            \toprule
              & L    & R    \\
            \midrule
            U & $x'$ & $x'$ \\
            D & $x$  & $x$  \\
            \bottomrule
        \end{tabular}
        \vspace{0.5em} \\
        G & G' \\
    \end{tabular}
    \end{center}
    Because $\varphi(x) = \varphi(x')$, we must have $\Phi(G) = \Phi(G')$ and therefore $g_i^k(G) = g_i^k(G')$.
    However, U is a $\zeta$-dominant action in $G$, and D is a $\zeta$-dominant action in $G'$.
    This means that, in one of the two games, $g_i^k$ plays a $\zeta$-dominant action with probability no more than $1/2$.
    Therefore, there exists a game where $f_i$ plays a $\zeta$-dominant action with probability no more than $1 - w_k/2$ regardless of $\zeta$, implying that $f_i$ is not strongly dominance responsive in the limit.
\end{proof}
Intuitively, the restrictions on the potential functions ensure that they either 
depend solely on the agent's own utility, making it impossible for the model to be sensitive to changes in the opponent's utility function,
or have arbitrarily wide level curves, making it impossible for the model to detect whether an action is dominated.

\clearpage
\section{Experiment Details}
\label{apx:experiments}

In this section, we describe additional details about our dataset and the training and evaluation procedure used in our experiments.

\paragraph{Data.}
We aggregated data from twelve experimental studies of human behavior in unrepeated normal-form games. 
The first ten were used in past work performing comprehensive evaluations of behavioral game theory models \cite{WLB2017, JAIR}.
The remaining two studies, by \citet{FL19} and \citet{CHW23}, ran large-scale experiments on Amazon Mechanical Turk.
\Cref{tab:all12} lists details of these datasets.

\begin{table*}
    \centering
    \begin{tabular}{llrr}
        \toprule
        Name   & Source        & Games & Obs. \\
        \midrule
        SW94   & \cite{SW94}   &    10 &  400 \\
        SW95   & \cite{SW95}   &    12 &  576 \\
        CGCB98 & \cite{CostaGomes2001} &    18 & 1296 \\
        GH01   & \cite{GH01}   &    10 &  500 \\
        CVH03  & \cite{CVH03}  &     8 & 2992 \\
        HSW01  & \cite{HSW01}  &    15 &  869 \\
        HS07   & \cite{HS07}   &    20 & 2940 \\
        SH08   & \cite{SH08}   &    18 & 1288 \\
        RPC08  & \cite{RPC08}  &    17 & 1210 \\
        CGW08  & \cite{CGW08}  &    14 & 1792 \\
        FL19   & \cite{FL19}   &   200 & 8250 \\
        CHW23  & \cite{CHW23}  &    24 & 4440 \\
        \midrule
        \textsc{All12} & Union of above & 366 & 26553 \\
        \bottomrule
    \end{tabular}
    \caption{Data used in our experiments. The ``Games'' column lists the number of unique games we included from each source, and ``Obs.'' columns list the total number of observations.}
    \label{tab:all12}
\end{table*}

\paragraph{Preprocessing.}
We preprocessed the data by standardizing the payoff matrices in each game, normalizing them to have zero mean and unit variance.
This preprocessing is useful for three reasons.
First, it controls for inflation during the three decades spanned by the data: the subjective meaning of a \$1 payment was different in 2023 than it was in 1993.
Second, it controls for differences between experiment payment protocols, as some experiments paid participants for their performance in all games, while others paid participants for one random game.
Third, it stabilizes model training by ensuring that gradients are on a relatively consistent scale.

\paragraph{QCH parameterization.}
We had some difficulty optimizing the precision and level distribution parameters of our QCH models directly.
We found that our models would quickly converge to high precisions, causing vanishing gradient problems that made it difficult for our optimizers to effectively improve the higher-level strategies; optimizing the log of the precision helped to solve this problem.
Similarly, for our QCH models with a finite number of levels, we optimized the logit level distribution, rather than optimizing the level probabilities directly.

\paragraph{Normalizing potential functions.}
We also had some difficulty optimizing the potential functions in the ElementaryNet models with learned potentials.
The problem was that when the potential function's coefficients $\theta_x$ and $\theta_y$ were too large or too small, the resulting potentials had a variance much larger or smaller than 1. 
This caused the response functions to have inconsistent gradients and ultimately led to poor performance.
To solve this problem, we divided the coefficients by $\sqrt{\theta_x^2 + \theta_y^2}$ before applying the potential function, ensuring that the output potentials have a consistent scale.
We found that this normalization step significantly improved performance.

\paragraph{Objective function.}
Following the recommendation of \citet{dEon2024}, we optimized the squared L2 error of the model's predictions.
Concretely, let $G = (A, u)$ be a game, let $\bar p(G) \in \Delta(A_1)$ be the empirical distribution of actions that subjects played when acting as player 1, and let $f(G)$ be the model's predicted distribution.
Then, the squared L2 loss on this game is
\begin{align*}
    L(f(G), \bar p(G)) 
    &= \Vert f(G) - \bar p(G) \Vert_2^2 \\
    &= \sum_{a \in A_1} (f(G)_a - \bar p(G)_a)^2.
\end{align*}
We aggregated losses across multiple games by taking a weighted average, where each game was weighted in proportion to the number of observations it had in the dataset.

\paragraph{Hyperparameters.}
To evaluate a single model on a single split of the data into a training, validation, and test set, we trained the model with a number of different hyperparameter settings, varying the initial settings of the QCH models' parameters, the amount of L1 regularization applied to the neural networks' weights, and the dropout probability applied to the neural networks.
These hyperparameters are detailed in \Cref{tab:hyperparams}.
For each training run, we ran Adam~\cite{Adam} for 50,000 epochs with a learning rate of $3 \cdot 10^{-4}$, using the projected gradient algorithm to impose simplex constraints where necessary.

\begin{table*}
    \centering
    \begin{tabular}{rcc}
        \toprule
        Hyperparameter & ElementaryNet & GameNet \\ 
        \midrule
        Initial QCH precision & \multicolumn{2}{c}{$\text{Log-Uniform}(0.03, 0.3)$} \\
        Initial QCH Poisson rate & \multicolumn{2}{c}{$\text{Uniform}(0.5, 1.5)$} \\
        Initial QCH level distribution logits & \multicolumn{2}{c}{$\text{Normal}(0, 1)$} \\
        Neural net weights L1 coefficient & $\{10^{-4}, 10^{-5}, 10^{-6}\}$ & $\{10^{-3}, 10^{-4}, 10^{-5}\}$ \\
        Dropout probability & $\{ 0, 0.01, 0.02, 0.05 \}$ & $\{0, 0.01, 0.02, 0.05, 0.1\}$ \\
        \bottomrule 
    \end{tabular}
    \caption{Hyperparameters used for model training. We ran all combinations of hyperparameter values listed for L1 and dropout twice, sampling random QCH parameters for each run independently from the listed distributions.}
    \label{tab:hyperparams}
\end{table*}

\paragraph{Computational resources.} 
We found that our training procedure was not dramatically accelerated by GPUs, as our dataset and models are relatively small for modern deep learning standards.
We thus chose to train our models on CPUs, as we had access to more CPU resources than GPUs.
We ran our experiments on a shared computer cluster having 14,672 CPUs, using a total of 40 CPU-years of computation time. 

\clearpage
\section{Additional Results}
\label{apx:results}
In Section 5, rather than directly reporting the loss of each model, we reported differences between models' losses.
This increased statistical power by reducing variance caused by differences in the data split used to evaluate each replication of each model.
However, this statistical power came at the expense of making it harder to gauge each model's performance in absolute terms.
In this section, we provide analogous figures showing absolute test losses: 
\begin{itemize}
  \item \Cref{fig:results-absolute-performance} compares GameNet and ElementaryNet models with a variety of neural net widths and depths, and a QCHp strategic model.
  \item \Cref{fig:results-absolute-ablation-noqch} compares ElementaryNet models with 1 to 3 hidden layers and no strategic model.
  \item \Cref{fig:results-absolute-ablation-qch} compares ElementaryNet models with alternative strategic models (QCH1, QCH2, and QCH3).
  \item \Cref{fig:results-absolute-ablation-potentials} compares ElementaryNet models with alternative potential functions: one with a single potential fixed to the ``own'' potential function, and one with all four of the fixed potential functions that we defined.
\end{itemize}

\begin{figure}
    \centering
    \includegraphics[width=\columnwidth]{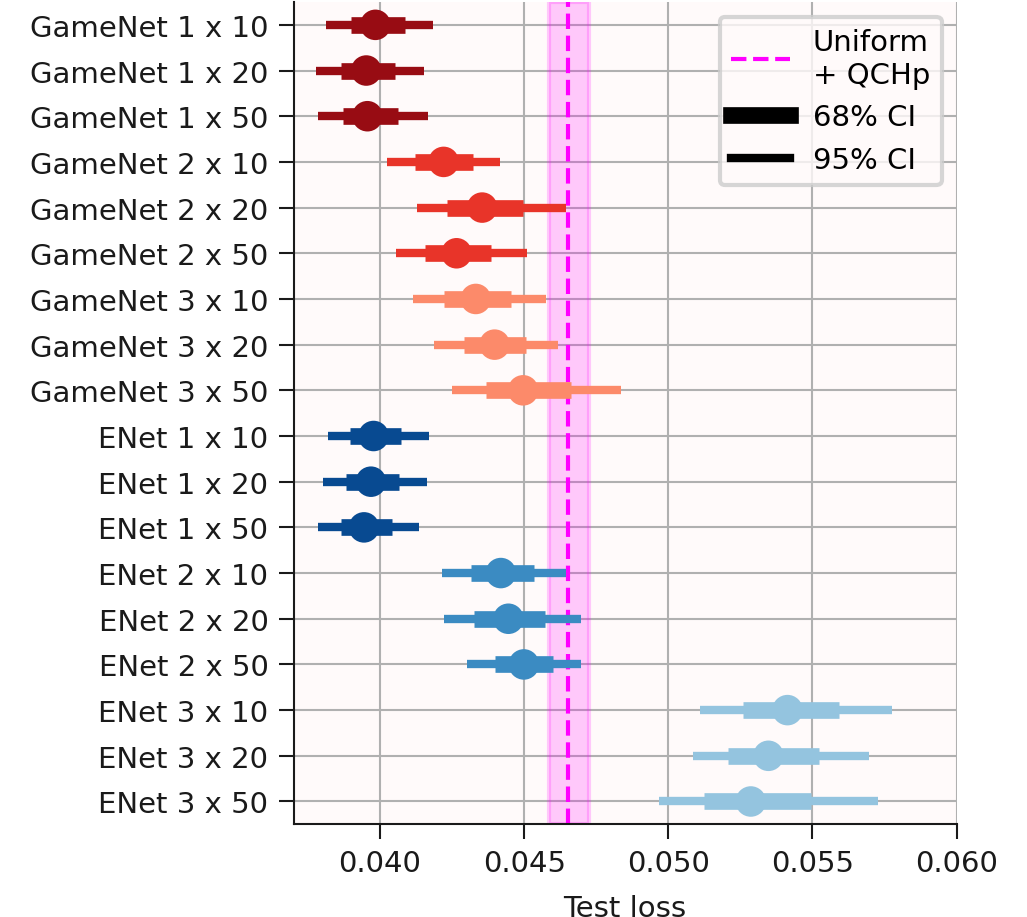}
    \caption{ElementaryNet and GameNet level-0 models trained with a QCH-Poisson strategic model.}
    \label{fig:results-absolute-performance}
\end{figure}

\begin{figure}
    \centering
    \includegraphics[width=\columnwidth]{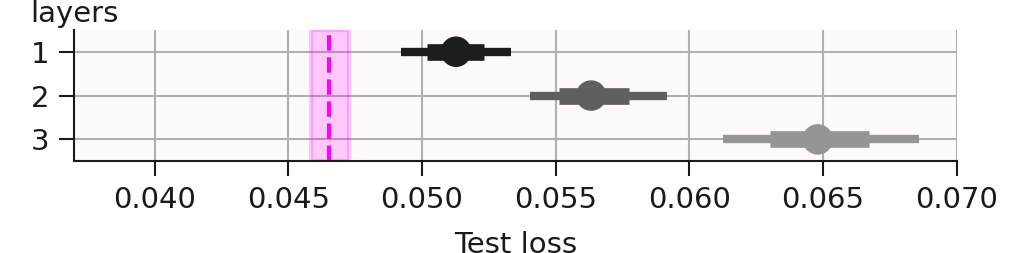}
    \caption{ElementaryNet models with no QCH model.}
    \label{fig:results-absolute-ablation-noqch}
\end{figure}

\begin{figure}
    \centering
    \includegraphics[width=\columnwidth]{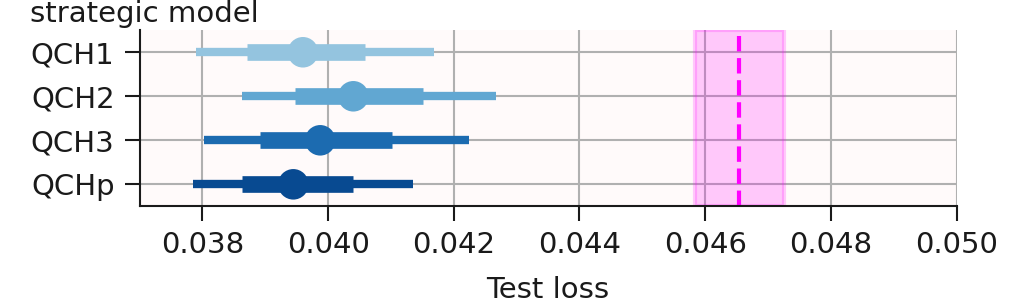}
    \caption{ElementaryNet level-0 models with various QCH models.}
    \label{fig:results-absolute-ablation-qch}
\end{figure}

\begin{figure}
    \centering
    \includegraphics[width=\columnwidth]{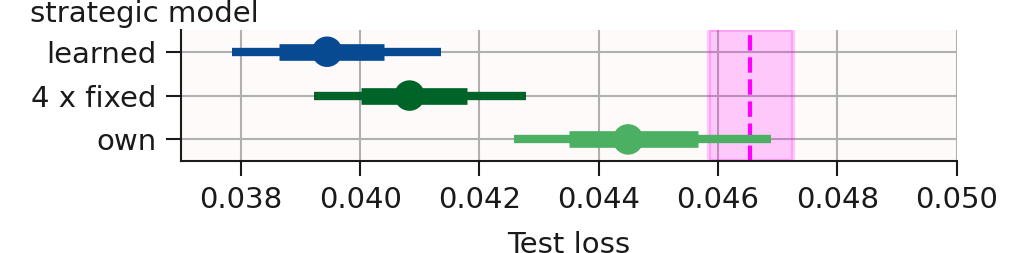}
    \caption{ElementaryNet level-0 models with various potential functions.}
    \label{fig:results-absolute-ablation-potentials}
\end{figure}

\end{document}